\documentclass{article}

\usepackage[preprint, nonatbib]{neurips_2021}

\usepackage[utf8]{inputenc} 
\usepackage[T1]{fontenc}    
\usepackage{hyperref}       
\usepackage{url}            
\usepackage{booktabs}       
\usepackage{amsfonts,amsmath,bm,amsthm,bbm}       
\usepackage{nicefrac}       
\usepackage{microtype}      
\usepackage{xcolor}         
\usepackage{graphicx}
\usepackage{float}
\usepackage{algorithm}
\usepackage{algorithmic}
\usepackage{subfig}
\usepackage{multirow}

\newtheorem{theorem}{Theorem}
\newtheorem{lemma}{Lemma}
\newtheorem{definition}{Definition}

\title{Multimodal Reward Shaping for Efficient Exploration in Reinforcement Learning}

\author{%
	Mingqi Yuan\raisebox{0pt}[0pt][0pt]{\textsuperscript{1,2}} \qquad Man-on Pun\raisebox{0pt}[0pt][0pt]{\textsuperscript{1,2}} \qquad Dong Wang\raisebox{0pt}[0pt][0pt]{\textsuperscript{1}} \qquad Yi Chen\raisebox{0pt}[0pt][0pt]{\textsuperscript{1,2}}\qquad Haojun Li\raisebox{0pt}[0pt][0pt]{\textsuperscript{1,2}} 
	\\\\
	\raisebox{0pt}[0pt][0pt]{\textsuperscript{1}}The Chinese University of Hong Kong, Shenzhen \\
	\raisebox{0pt}[0pt][0pt]{\textsuperscript{2}}Shenzhen Research Institute of Big Data \\
	\texttt{\{mingqiyuan, haojunli\}@link.cuhk.edu.cn} \\
	\texttt{\{simonpun,yichen,wangdong\}@cuhk.edu.cn}
}

\begin{document}
	
	\maketitle

\begin{abstract}
	Maintaining the long-term exploration capability of the agent remains one of the critical challenges in deep reinforcement learning. A representative solution is to leverage reward shaping to provide intrinsic rewards for the agent to encourage exploration. However, most existing methods suffer from vanishing intrinsic rewards, which cannot provide sustainable exploration incentives. Moreover, they rely heavily on complex models and additional memory to record learning procedures, resulting in high computational complexity and low robustness. To tackle this problem, entropy-based methods are proposed to evaluate the global exploration performance, encouraging the agent to visit the state space more equitably. However, the sample complexity of estimating the state visitation entropy is prohibitive when handling environments with high-dimensional observations. In this paper, we introduce a novel metric entitled Jain's fairness index to replace the entropy regularizer, which solves the exploration problem from a brand new perspective. In sharp contrast to the entropy regularizer, JFI is more computable and robust and can be easily applied generalized into arbitrary tasks. Furthermore, we leverage a variational auto-encoder model to capture the life-long novelty of states, which is combined with the global JFI score to form multimodal intrinsic rewards. Finally, extensive simulation results demonstrate that our multimodal reward shaping (MMRS) method can achieve higher performance than other benchmark schemes. Our code is available at our GitHub website \footnote{https://github.com/yuanmingqi/MMRS}.
\end{abstract}

\section{Introduction}
Reinforcement learning (RL) aims to learn the optimal policy that maximizes the long-term expected return \cite{sutton2018reinforcement}. To that end, the agent is required to keep exploring the environment and visit all possible state-action pairs infinitely \cite{watkins1992q}. However, many existing RL algorithms suffer from inadequate exploration mechanisms, in which the agent can only linger in a restricted area throughout the whole learning procedure. As a result, the policy always prematurely falls into local optima after finite steps and never improves again \cite{stadie2015incentivizing}. To maintain exploration across episodes, a simple approach is to employ stochastic policies such as $\epsilon$-greedy policy and Boltzmann exploration \cite{lecun2015deep}, which randomly select all the possible actions with a non-zero probability in each state. For continuous control tasks, an additional noise item is added to the action to realize limited exploration. Such techniques are prone to learn the optimal policy eventually in the tabular setting, but the learning overhead is prohibitive when handling complex environments with high-dimensional observations. 

To cope with exploration problems in complex tasks, the reward shaping approach is leveraged to form an additional reward mechanism to improve exploration. More specifically, recent approaches proposed to provide intrinsic rewards for the agent to assess its exploration performance \cite{dayan2002reward}. In sharp contrast to the extrinsic rewards given by the environment explicitly, intrinsic rewards represent the inherent learning motivation or curiosity of the agent, which are difficult to characterize and evaluate \cite{singh2005intrinsically}. Many pioneering works have been devoted to realizing computable intrinsic reward modules, and they can be broadly categorized into novelty-based and prediction error-based approaches \cite{xu2017study, csimcsek2006intrinsic, burda2018exploration, lee2019efficient}. For instance, \cite{strehl2008analysis, ostrovski2017count, bellemare2016unifying} employed a state visitation counter to evaluate the novelty of states. Since it is challenging to perform counting in environments with high-dimensional observations, the pseudo-count method is proposed to approximate the actual count results. Moreover, a higher bonus will be assigned to those infrequently-seen states, incentivizing the agent to revisit surprising states and increasing the probability of learning better policy. Methods in \cite{pathak2017curiosity, yu2020intrinsic, yuan2021hybrid, stadie2015incentivizing} followed the second idea, in which the prediction error of a dynamic model is utilized as intrinsic rewards. Given an observed transition, an attendant model was designed to predict a successor state based on the current state-action pair. Then the intrinsic reward is computed as the Euclidean distance between the predicted successor state and the true successor state. In particular, \cite{burda2018large} attempted to perform RL only using the second intrinsic rewards, in which the agent could also achieve considerable performance in many experiments.

All the methods above produce vanishing intrinsic rewards, i.e., the intrinsic rewards will decrease with visits \cite{ecoffet2019go}. Once the clip of intrinsic rewards is empty, the agent will have no additional motivation to explore the environment further. To obtain the long-term exploration ability, \cite{badia2020never} proposed a never-give-up (NGU) framework that learns mixed intrinsic rewards consist of episodic and life-long state novelty. NGU evaluates the episodic state novelty through a slot-based memory and pseudo-count method \cite{bellemare2016unifying}, encouraging the agent to visit more distinct states in each episode. Since the memory will be wiped at the beginning of the episode, the intrinsic rewards will not decay with the training process. Meanwhile, NGU further introduced a random network distillation (RND) module to capture the life-long novelty of states \cite{burda2018exploration}. By controlling the learning rate of RND, the life-long intrinsic reward can prevent the agent from visiting familiar states more smoothly. NGU suffers from complicated architecture and high computational complexity, making it difficult to be generalized into arbitrary tasks. A more straightforward framework entitled rewarding impact-driven exploration (RIDE) is proposed in \cite{raileanu2020ride}. RIDE inherits the inverse-forward pattern of \cite{pathak2017curiosity}, in which two discriminative models are leveraged to reconstruct the transition process. Then the Euclidean distance between two consecutive encoded states is utilized as the intrinsic reward, which encourages the agent to take actions that result in more state changes. Moreover, RIDE uses episodic state visitation counts to discount the generated rewards, preventing the agent from staying at states that lead to large embedding differences and avoiding the television dilemma reported in \cite{savinov2018episodic}.

However, both NGU and RIDE pay excessive attention to single states, which cannot reflect the global exploration extent. Moreover, the methods above have poor mathematical interpretability and rely on attendant models heavily. To circumvent this problem, \cite{islam2019entropy} proposed to maximize the entropy of state visitation distribution, forcing the agent to visit the state space more equitably. \cite{islam2019entropy} first proved that such distribution could be estimated via importance sampling before using a variational auto-encoder (VAE) model \cite{kingma2013auto} to estimate its entropy. Given a transition, the VAE accepts the parameters of policy as input and outputs a reconstructed state. Therefore, maximizing the reconstruction error of VAE is equivalent to maximizing the state visitation entropy. Finally, the policy and VAE model can be updated together through the policy gradient method. In particular, \cite{zhang2018dissection} further expanded the Shannon entropy regularizer into R\'enyi Entropy to adapt to arbitrary tasks. To realize an efficient and stable entropy estimate, \cite{seo2021state} proposed a random encoder for efficient exploration (RE3) framework that requires no representation learning. The observations are collected and encoded in each episode using a fixed deep neural network (DNN). Then a $k$-nearest neighbor estimator \cite{singh2003nearest} is leveraged to estimate the entropy of SVD. Simulations results demonstrated that RE3 improved the sample efficiency of both model-free and model-based RL algorithms. However, the estimation error and sample complexity of all the methods above grow exponentially with the state space size, and an imperfect estimation will inevitably mislead the policy learning. 

Inspired by the discussions above, it is non-trivial to find a more straightforward and computable method to estimate or replace the entropy regularizer. Moreover, we consider combining the state novelty and the entropy regularizer to form multimodal intrinsic rewards, evaluating the exploration performance more precisely. In this paper, we propose multimodal reward shaping (MMRS), a model-free, fairness-driven, and generative-model empowered method for providing high-quality intrinsic rewards. Our main contribution are summarized as follows:
\begin{itemize}
	\item We first dived into the sample complexity of the entropy-based intrinsic rewards, and formally proved a bound of the estimation error. After that, we reanalyzed and proposed a new perspective to address the exploration problem. In particular, a novel metric entitled Jain's fairness index was introduced to replace the entropy-regularizer, and we proved the utility equivalence between the two metrics. Furthermore, we discussed the practical employment of JFI both in tabular settings and environments with high-dimensional observations.
	
	\item Since Jain's fairness index evaluates the global exploration performance, the life-long state novelty was leveraged to form multimodal intrinsic rewards. In particular, we used a VAE model to perform state embedding and capture the life-long novelty of states. Such a method requires no additional memory and avoids overfitting, which is more efficient and robust than RND.
	
	\item Finally, extensive simulations are performed to compare MMRS with existing similar methods. We first validated that MMRS can overcome the problem of vanishing intrinsic rewards and maintain long-term exploration ability. Then the MMRS was tested both in discrete and continuous control tasks, in which the selected games have complicated state space and action space. Numerical results demonstrated that MMRS outperforms the benchmark methods with simpler architecture and higher robustness. 
\end{itemize}

\section{Problem Formulation}\label{section:pf}
In this paper, we study the RL problem that considers the Markov decision process (MDP) defined as a tuple $\mathcal{M}=\langle \mathcal{S},\mathcal{A},\mathcal{T},r,\rho(\mathbf{s}_0),\gamma\rangle$ \cite{sutton2018reinforcement}, in which $\mathcal{S}$ is the state space, $\mathcal{A}$ is the action space, $\mathcal{T}(\mathbf{s}'|\mathbf{s},\mathbf{a})$ is the transition probability, $r(\mathbf{s},\mathbf{a},\mathbf{s}'):\mathcal{S}\times\mathcal{A}\times\mathcal{S}\rightarrow\mathbb{R}$ is the reward function, $\rho(\mathbf{s}_0)$ is the initial state distribution, and $\gamma\in(0,1]$ is a discount factor. Note that the reward function here conditions on a full transition $(\mathbf{s},\mathbf{a},\mathbf{s}')$. Furthermore, we denote $\pi(\mathbf{a}|\mathbf{s})$ as the policy of agent, which observes the state of environment before choosing an action from the action space. Equipped with these basic settings, we formally define the objective of RL.

\begin{definition}[RL]
	Given MDP $\mathcal{M}$, the objective of RL is to find the optimal policy $\pi^{*}$ that maximizes the expected discounted return:
	\begin{equation}\label{eq:rl objective}
		\pi^{*}=\underset{\pi\in\Pi}{\rm argmax}\;\mathbb{E}_{\tau\sim\pi}\sum_{t=0}^{T-1}\gamma^{t}r_{t}(\mathbf{s}_t,\mathbf{a}_t, \mathbf{s}_{t+1}),
	\end{equation}
	where $\Pi$ is the set of all stationary policies, and $\tau=(\mathbf{s}_{0},\mathbf{a}_{0},\dots,\mathbf{a}_{T-1},\mathbf{s}_{T})$ is the trajectory generated by the policy. 
\end{definition}

The following sections first analyze the complexity when using the entropy-based method to promote exploration. Then a new perspective is proposed to evaluate the exploration performance, and an alternative method is designed for replacing the entropy regularizer.

\section{State Visitation Entropy}
In this paper, we aim to improve the exploration of state space, in which the agent is expected to visit as many distinct states as possible within limited learning procedure. To evaluate the exploration extent of state space, we define the following state visitation distribution \cite{kakade2003sample}:
\begin{equation}\label{eq:svd}
	d^{\pi}(\mathbf{s})=(1-\gamma)\sum_{t=0}^{\infty}\gamma^{t}P(S_{t}=\mathbf{s}),\forall \mathbf{s}\in\mathcal{S},
\end{equation}
where $P(\cdot)$ denotes the probability, $S_t$ is the random variable of state at step $t$. Therefore, improving exploration requires the agent to provide equitable visitation probability for all states. Mathematically, it is equivalent to maximizing the following state visitation entropy \cite{1957Information}:
\begin{equation}\label{eq:entropy}
	H(d^{\pi})=-\int_{\mathbf{s}\in\mathcal{S}}d^{\pi}(\mathbf{s})\log d^{\pi}(\mathbf{s}) {\rm d}\mathbf{s}.
\end{equation}

To compute the entropy, we first need to estimate the state visitation distribution defined in Eq.~\eqref{eq:svd}. However, it is challenging to sample from such a distribution because it is only a theoretical construction. To address the problem, we modify the state distribution $P(\mathbf{s})$ through importance sampling to approximate the state visitation distribution. Given a trajectory $\tau=(\mathbf{s}_{0},\mathbf{a}_{0},\dots,\mathbf{a}_{T-1},\mathbf{s}_{T})$ generated by policy $\pi$, a reasonable estimate \cite{durrett2019probability} of $P(\mathbf{s})$ can be computed as:
\begin{equation}\label{eq:estimate svd}
	\hat{P}(\mathbf{s})=\frac{1}{T}\sum_{t=0}^{T}\mathbbm{1}(S_t=\mathbf{s}),
\end{equation}
where $\mathbbm{1}(\cdot)$ is the indicator function, and the weight of each sample is $\frac{1}{T}$. The following lemma indicates the sample complexity of such estimation:
\begin{lemma}\label{lemma:estimate svd}
	Given a trajectory $\tau=(\mathbf{s}_{0},\mathbf{a}_{0},\dots,\mathbf{a}_{T-1},\mathbf{s}_{T})$, assume estimating the state distribution following Eq.~\eqref{eq:estimate svd}, then with probability at least $1-\delta$, $\exists\,c\in \mathbb{R}^{+}$ such that:
	\begin{equation}
		\Vert P(\mathbf{s})-\hat{P}(\mathbf{s}) \Vert_{1} \leq c\sqrt{\frac{|\mathcal{S}|\log(1/\delta)}{T}},
	\end{equation}
	where $|\mathcal{S}|$ is the cardinality of state space.
\end{lemma}
\begin{proof}
	See proof in Appendix \ref{proof:estimate svd}.
\end{proof}

Given the state space, Lemma \ref{lemma:estimate svd} demonstrates that the estimation error of $P(\mathbf{s})$ decays with the speed of $O(\sqrt{\frac{1}{T}})$. Then an importance sampling with a weight of $(1-\gamma)\gamma^{t}$ is performed on $P(\mathbf{s})$ to yield:
\begin{equation}
		\frac{(1-\gamma)}{T}\sum_{t=0}^{T}\gamma^{t}\mathbbm{1}(S_t=\mathbf{s})\stackrel{(a)}{=}
		(1-\gamma)\sum_{t=0}^{T}\gamma^{t}P(S_{t}=\mathbf{s}|S_{0})
		\stackrel{(b)}{\approx} d^{\pi}(\mathbf{s}),
\end{equation}
where $(a)$ follows the fact that $P(S_{t}=\mathbf{s}|S_{0})=\frac{\mathbbm{1}(S_t=\mathbf{s})}{T}$, and $(b)$ is the approximation using a finite truncation of infinite horizon trajectory. Next, we estimate the state visitation entropy using Monte-Carlo sampling method \cite{shapiro2003monte}:
\begin{equation}\label{eq:estimate entropy}
	\hat{H}(d^{\pi})=-\frac{1}{T}\sum_{t=0}^{T}\log d^{\pi}(\mathbf{s}).
\end{equation}

Similar to the estimation of the state visitation distribution, we can formally prove that:
\begin{lemma}\label{lemma:estimate entropy}
	Given a trajectory $\tau=(\mathbf{s}_{0},\mathbf{a}_{0},\dots,\mathbf{a}_{T-1},\mathbf{s}_{T})$, assume estimating the entropy following Eq.~\eqref{eq:estimate entropy}, then with probability at least $1-\delta$, it holds:
	\begin{equation}
		|\hat{H}(d^{\pi})-H(d^{\pi})|\leq \log |\mathcal{S}| \sqrt{\frac{\log (2/\delta)}{2T}}.
	\end{equation}
\end{lemma}
\begin{proof}
	See proof in Appendix \ref{proof:estimate entropy}.
\end{proof}

Equipped with Lemma \ref{lemma:estimate svd} and Lemma \ref{lemma:estimate entropy}, the sample complexity of estimating the state visitation entropy can be approximated as the multiplication:
\begin{equation}\label{eq:convergence rate}
	O(\log |\mathcal{S}|\cdot \sqrt{\frac{|\mathcal{S}|}{2T^{2}}}).
\end{equation}
Given a state space, Eq.~\eqref{eq:convergence rate} indicates that the estimation error decays with the speed of $O(\frac{1}{T})$, which is a considerable estimation in tabular setting. However, such estimation may produce prohibitive variance and mislead the policy learning when handling the environments with high-dimensional observations.

\section{Fairness-Driven Exploration}\label{section:fde}
To replace the complex entropy regularizer, we introduce Jain's fairness index (JFI) \cite{jain1999throughput} to evaluate the global exploration performance. JFI is a count-based metric that is first leveraged to assess the allocation fairness in the radio resources management. Given resources and allocation objects, JFI can accurately reflect the fairness difference between different allocation schemes. Our key idea is to regard the limited visitation steps as the resources to be allocated and set all the possible states as allocation objects. Therefore, higher visitation fairness indicates that more distinct states are visited, while over-centralized visitation of few states will be punished. Since maximizing the state visitation entropy also aims to realize equitable visitation, it is feasible to replace the entropy regularizer with JFI based on their utility equivalence. The following definition formulates the JFI for state visitation:

\begin{definition}[JFI for state visitation]\label{def:jfi}
	Given an episode trajectory $\tau=(\mathbf{s}_{0},\mathbf{a}_{0},\dots,\mathbf{a}_{T-1},\mathbf{s}_{T})$ generated by policy $\pi$, denote by $c(\mathbf{s}, \tau)$ the state visitation counter and $\mathcal{C}_{\tau}=\{c(\mathbf{s}, \tau)\}_{\mathbf{s}\in \mathcal{S}}$, the visitation fairness based on JFI can be computed as:
	\begin{equation}\label{eq:jfi}
		J(\mathcal{C}_{\tau})=\frac{\big[\sum_{\mathbf{s}\in\mathcal{S}}c(\mathbf{s}, \tau)\big]^{2}}{|\mathcal{S}|\sum_{\mathbf{s}\in\mathcal{S}}\big[c(\mathbf{s}, \tau)\big]^{2}},
	\end{equation}
	where $J(\mathcal{C}_{\tau})$ ranges form $1/|\mathcal{S}|$ (worst case) to $1$ (best case), and it is maximum when $c(\mathbf{s})$ gets the same value for all the states.
\end{definition}

Given a transition $(\mathbf{s}_{t},\mathbf{a}_{t},\mathbf{s}_{t+1})$, we define the following shaping function to evaluate the gain performance of exploration when transiting from $\mathbf{s}_{t}$ to $\mathbf{s}_{t+1}$:
\begin{equation}\label{eq:global irs}
	G(\mathbf{s}_{t},\mathbf{s}_{t+1})=\gamma J(\mathcal{C}_{\tau_{t+1}})-J(\mathcal{C}_{\tau_{t}}),
\end{equation}
where $\tau_{t}=(\mathbf{s}_{0},\mathbf{a}_{0},\dots,\mathbf{a}_{t-1},\mathbf{s}_{t})$ is the sub-episode trajectory. Equipped with Eq.~\eqref{eq:global irs}, the following theorem formally proves the utility equivalence between JFI and state visitation entropy.

\begin{theorem}[Consistency]\label{theorem:consistency}
	Given a trajectory $\tau=(\mathbf{s}_{0},\mathbf{a}_{0},\dots,\mathbf{a}_{T-1},\mathbf{s}_{T})$, using Eq.~\eqref{eq:global irs} as reward shaping function is equivalent to maximizing the state visitation entropy when $T\rightarrow\infty$.
\end{theorem}
\begin{proof}
	Recall the optimal condition of JFI, it holds $c(\tau,\mathbf{s})=T/|\mathcal{S}|,\forall \mathbf{s}\in\mathcal{S}$, such that $d^{\pi}(\mathbf{s})=1/|\mathcal{S}|$. This concludes the proof.
\end{proof}

\begin{figure}[h]
	\centering
	\includegraphics[width=0.7\linewidth]{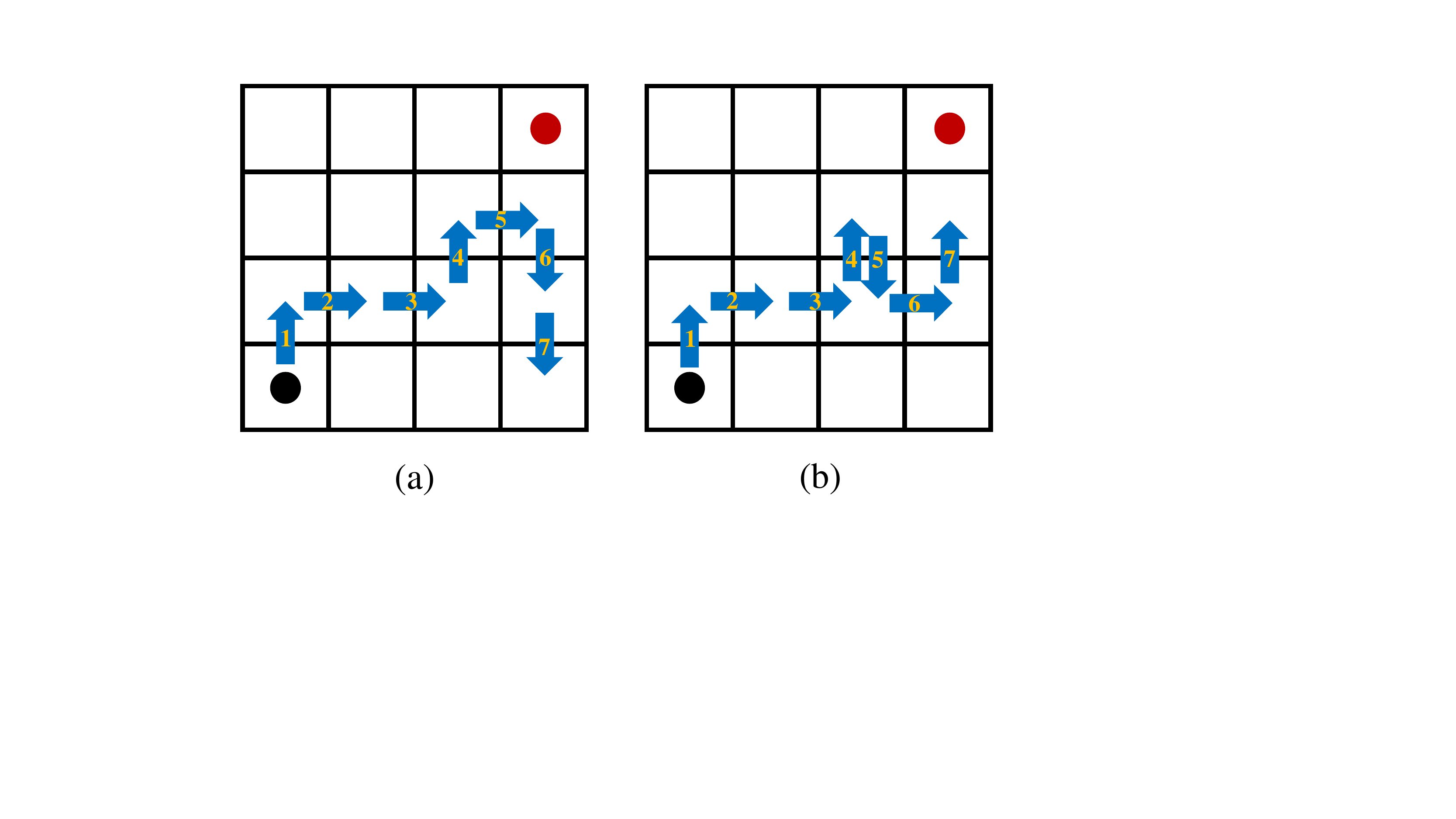}
	\caption{Two exploration trajectories in the GridWorld game, where the black node and red node denote the start and end, and integers denote the index of actions.}
	\label{fig:gridworld}
\end{figure}

Furthermore, we employ a representative example to demonstrate the usage and advantage of JFI. Fig.~\ref{fig:gridworld} demonstrates two trajectories generated by the agent when interacting with the GridWorld game. Fig.~\ref{fig:gridworld}(a) makes better exploration than Fig.~\ref{fig:gridworld}(b) because it visits more novel states within same steps. Based on Eq.~\eqref{eq:jfi}, the variation of visitation fairness can be computed as follows: 
\begin{equation}
	\begin{aligned}
		(a):\frac{1}{16}\stackrel{\nearrow}{\Longrightarrow}\frac{2}{16}\stackrel{\nearrow}{\Longrightarrow}\frac{3}{16}\stackrel{\nearrow}{\Longrightarrow}\frac{4}{16}\stackrel{\nearrow}{\Longrightarrow}\frac{5}{16}\stackrel{\nearrow}{\Longrightarrow}\frac{6}{16}\stackrel{\nearrow}{\Longrightarrow} \frac{7}{16}\stackrel{\nearrow}{\Longrightarrow}\frac{8}{16}, \\
		(b):\frac{1}{16}\stackrel{\nearrow}{\Longrightarrow}\frac{2}{16}\stackrel{\nearrow}{\Longrightarrow}\frac{3}{16}\stackrel{\nearrow}{\Longrightarrow}\frac{4}{16}
		\stackrel{\nearrow}{\Longrightarrow}\frac{5}{16}\stackrel{\searrow}{\Longrightarrow}\frac{9}{32}\stackrel{\nearrow}{\Longrightarrow}\frac{49}{144}\stackrel{\nearrow}{\Longrightarrow}\frac{2}{5},
	\end{aligned}
\end{equation}
where $\nearrow$ represents an increase and $\searrow$ represents a decrease. It is obvious that keeping exploration will increase the JFI, and visiting the known region repeatedly will be punished. Finally, Fig.~\ref{fig:gridworld}(a) obtains higher visitation fairness when compared with Fig.~\ref{fig:gridworld}(b). In sharp contrast to the entropy regularizer, the JFI is very sensitive to the change of exploratory situation, and it has higher computation efficiency.

\textbf{JFI for infinite state space.} It is easy to compute JFI by rapid and straightforward counting in tabular settings. However, such operation is challenging when handling environments with infinite state space, because episodic exploration can only visit a few parts of the state space, resulting in a fixed JFI score of the worst case. To address that problem, we first set visited states of each episode as its temporary state space before leveraging the $k$-means clustering \cite{macqueen1967some} to discretize the densely-distributed states to make them countable.

%

Given an episode trajectory $\tau=(\mathbf{s}_{0},\mathbf{a}_{0},\dots,\mathbf{a}_{T-1},\mathbf{s}_{T})$, $k$-means clustering shatters the visited states into $k$ sets $\boldsymbol{\mathcal{\tilde{S}}} =\{\mathcal{\tilde{S}}_{1},\dots,\mathcal{\tilde{S}}_{k}\}$ to minimize the sum of within-cluster distance. Formally, the algorithm aims to optimize the following objective:
\begin{equation}
	\underset{\boldsymbol{\mathcal{\tilde{S}}}}{\rm argmin}\sum_{i=1}^{k}\sum_{\mathbf{s}\in\mathcal{\tilde{S}}_{i}}\Vert \mathbf{s}-\bm{\mu}_{i} \Vert_{2}^{2},
\end{equation}
where $\bm{\mu}_{j}$ is the mean of samples in $\mathcal{\tilde{S}}_{j}$. Denote by $\mathbf{l}\in\mathcal{K}=\{1,\dots,k\}$ the label result of the visited states, assume the trajectory is labeled as $\tilde{\tau}=(\mathbf{l}_{0},\mathbf{a}_{0},\dots,\mathbf{a}_{T-1},\mathbf{l}_{T})$, then Eq.~\eqref{eq:jfi} is rewritten as:
\begin{equation}
	J(\mathcal{C}_{\tilde{\tau}})=\frac{\big[\sum_{\mathbf{l}\in\mathcal{K}}c(\mathbf{l},\tilde{\tau})\big]^{2}}{k\cdot\sum_{\mathbf{l}\in\mathcal{K}}\big[c(\mathbf{l},\tilde{\tau})\big]^{2}}.
\end{equation}

In practice, we perform clustering using the encoding fashion of the states to reduce variance and computation complexity. To realize efficient and robust encoding operation, a VAE model is bulit in the following section.

\section{Multimodal Reward Shaping}\label{section:mmrs}
\begin{figure*}[h]
	\centering
	\includegraphics[width=1.\linewidth]{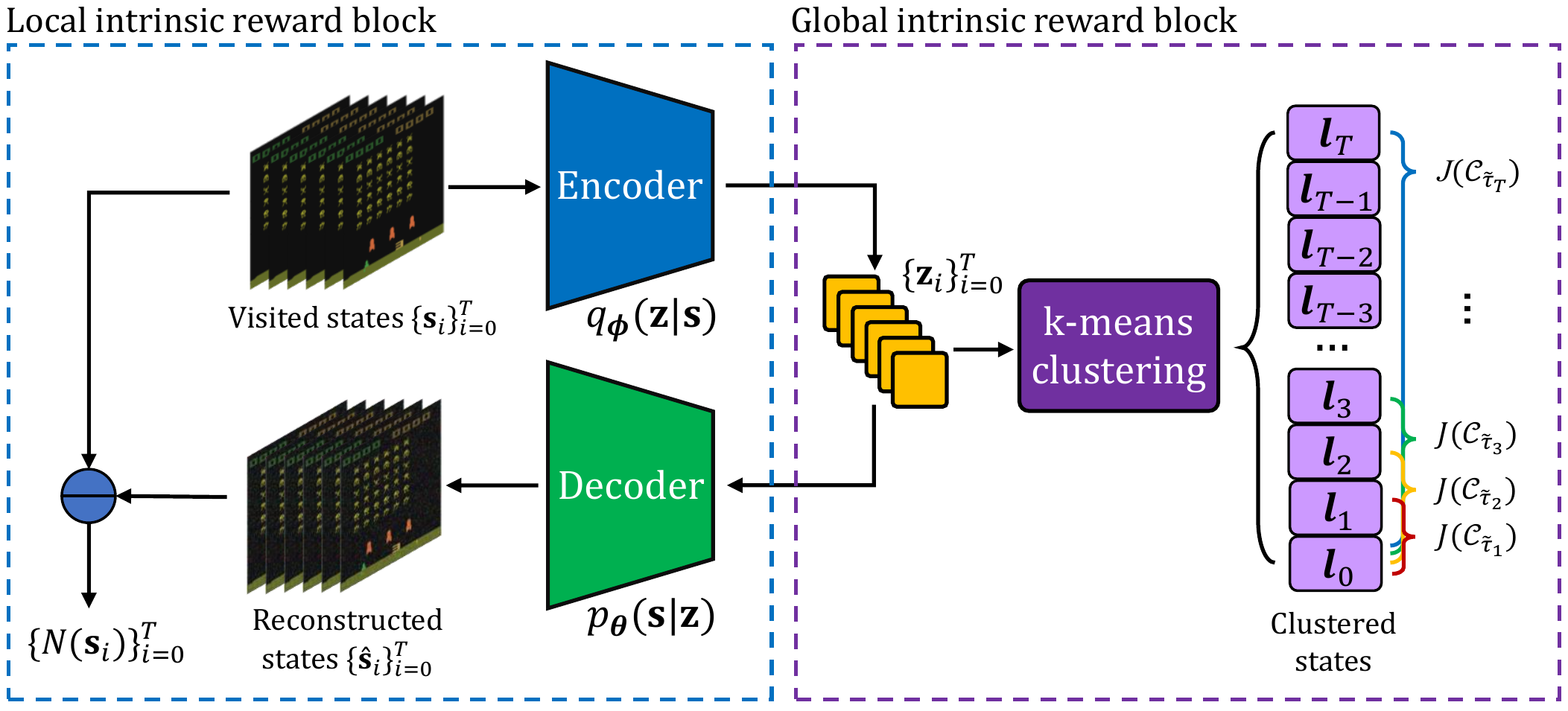}
	\caption{The overview of MMRS, where $\ominus$ denotes the Euclidean distance.}
	\label{fig:mmrs}
\end{figure*}

In this section, we propose an intrinsic reward module entitled MMRS that is model-free, fairness-driven, and generative-model empowered. As illustrated in Fig.~\ref{fig:mmrs}, MMRS is composed of two major modules, namely global intrinsic reward block and local intrinsic reward block. The former block evaluates the global exploration performance using JFI, while the latter traces the life-long novelty of states across episodes. Finally, the two kinds of exploration bonuses form multimodal intrinsic rewards for policy updates.

To capture the life-long state novelty, a general method is to employ an attendant model to record the visited states such as RND and ICM \cite{pathak2017curiosity}. However, the discriminative models suffer from overfitting and poor generalization ability. To overcome the problem, we propose to evaluate the state novelty using a VAE model, which is a powerful generative model based on Bayesian inference \cite{kingma2013auto}. A vanilla VAE has a recognition model and generative model, and they can be represented as a probabilistic \textit{encoder} and a probabilistic \textit{decoder}. We leverage the VAE to encode and reconstruct the state samples to calculate its life-long novelty. In particular, the output of the encoder can be leveraged to perform the clustering operation defined in Section \ref{section:fde}.

Denote by $q_{\bm \phi}(\mathbf{z}|\mathbf{s})$ the recognition model represented by a DNN with parameters ${\bm \phi}$, which accepts a state and encodes it into latent variables. Similarly, we represent the generative model $p_{\bm \theta}(\mathbf{s}|\mathbf{z})$ using a DNN with parameters ${\bm \theta}$, accepting the latent variables and reconstructing the state. Given a trajectory $\tau=(\mathbf{s}_{0},\mathbf{a}_{0},\dots,\mathbf{a}_{T-1},\mathbf{s}_{T})$, the VAE model is trained by minimizing the following loss function:
\begin{equation}\label{eq:vae loss}
	L(\mathbf{s}_{t};{\bm \phi},{\bm \psi})=-D_{{\rm KL}}\big(q_{\bm\phi}(\mathbf{z}|\mathbf{s}_{t})\Vert p_{\bm \theta}(\mathbf{z})\big)+\mathbb{E}_{q_{\bm \phi}(\mathbf{z}|\mathbf{s}_{t})}\big[\log p_{\bm \theta}(\mathbf{s}_{t}|\mathbf{z})\big],
\end{equation}
where $t=0,\dots,T$ and $D_{{\rm KL}}(\cdot)$ is the Kullback-Liebler (KL) divergence. For a visited state at step $t$, its life-long novelty is computed as:
\begin{equation}\label{eq:novelty}
	N(\mathbf{s}_{t})=\frac{1}{2}\Vert \Lambda(\mathbf{s}_{t})-\Lambda(\hat{\mathbf{s}}_{t}) \Vert_{2}^{2},
\end{equation}
where $\hat{\mathbf{s}}_{t}$ is the reconstructed state and $\Lambda(\cdot)$ is a normalization operator. This definition indicates that the infrequently-seen states will produce high reconstruction error, which motivates the agent to revisit and explore it further. Note that the VAE model will inevitably produce vanishing intrinsic rewards. However, we can employ a low and decaying learning rate for updates to control its decay rate.

\begin{algorithm}[h]
	\caption{Multimodal Reward Shaping}
	\label{algo:mmrs}
	\begin{algorithmic}[1]
		\STATE Initialize the policy network $\pi$, recognition model $q_{\bm \phi}$ and generative model $p_{\bm \theta}$;
		\STATE Set the coefficients $\lambda_{G}, \lambda_{L}$ and the number of clusters $k$;
		\FOR {episode $\ell=1,\dots,E$}
		\STATE Execute policy $\pi$ and collect the trajectory $\tau_{\ell}=(\mathbf{s}_{0},\mathbf{a}_{0},\dots,\mathbf{a}_{T-1},\mathbf{s}_{T})$;
		\IF {$|\mathcal{S}|\rightarrow\infty$}
		\STATE Use the encoder to encode the visited states and collect the corresponding latent variables $\{\mathbf{z}_{i}\}_{i=0}^{T}$;
		\STATE Perform $k$-means clustering to refine $\{\mathbf{z}_{i}\}_{i=0}^{T}$ and label them with $k$ integers;
		\ENDIF
		\STATE Calculate the multimodal intrinsic reward for each state-action pair in $\tau_{\ell}$:
		\begin{equation}\nonumber
		\tilde{r}(\mathbf{s},\mathbf{a},\mathbf{s}') = r(\mathbf{s},\mathbf{a},\mathbf{s}') + f(\mathbf{s},\mathbf{a},\mathbf{s}');
		\end{equation}
		
		\STATE Update $\pi$ with respect to the mixed rewards using any RL algorithms.
		\STATE Use the visited states $\{\mathbf{s}_{i}\}_{i=0}^{T}$ from $\tau_{\ell}$ to train the VAE model by minimizing the loss function defined in Eq.~\eqref{eq:vae loss};
		\ENDFOR
	\end{algorithmic}
\end{algorithm}

We refer to the JFI as global intrinsic rewards and state novelty as local intrinsic rewards, respectively. Equipped with the two kinds of exploration bonuses, we are ready to propose the following multimodal shaping function:
\begin{equation}\label{eq:mixed sf}
	f(\mathbf{s},\mathbf{a},\mathbf{s}')=\underbrace{\lambda_{G}\cdot G(\mathbf{s},\mathbf{s}')}_{\rm Global}+
	\underbrace{\lambda_{L}\cdot N(\mathbf{s})}_{\rm Local}
\end{equation}

where $\lambda_{G},\lambda_{L}$ are two weighting coefficients. Finally, the workflow of the MMRS is summarized in Algorithms \ref{algo:mmrs}.

\section{Experiments}
In this section, we evaluate our MMRS framework both on tabular setting and environments with high-dimensional observations. We compare MMRS with standard RL and three representative intrinsic reward methods for exploration, namely RE3, RIDE and RND. The brief introduction of these benchmark schemes can be found in Appendix \ref{appendix:benchmarks}. As for hyper-parameters setting, we only report the values of the best experiment results.

\subsection{Maze Games}
In this section, we first leverage a simple but representative example to highlight the effectiveness of fairness-driven exploration. We introduce a grid-based environment Maze2D implemented by \cite{matthew2016github}, depicted in Fig.~\ref{fig:maze}(a). The agent can take four actions, including left, right, up, and down and move a single position at a time. The goal of the agent is to find the shortest path from start to end. In particular, the agent can teleport from a portal to another portal of the same mark.

\begin{figure}[h]
	\centering
	\includegraphics[width=1.\linewidth]{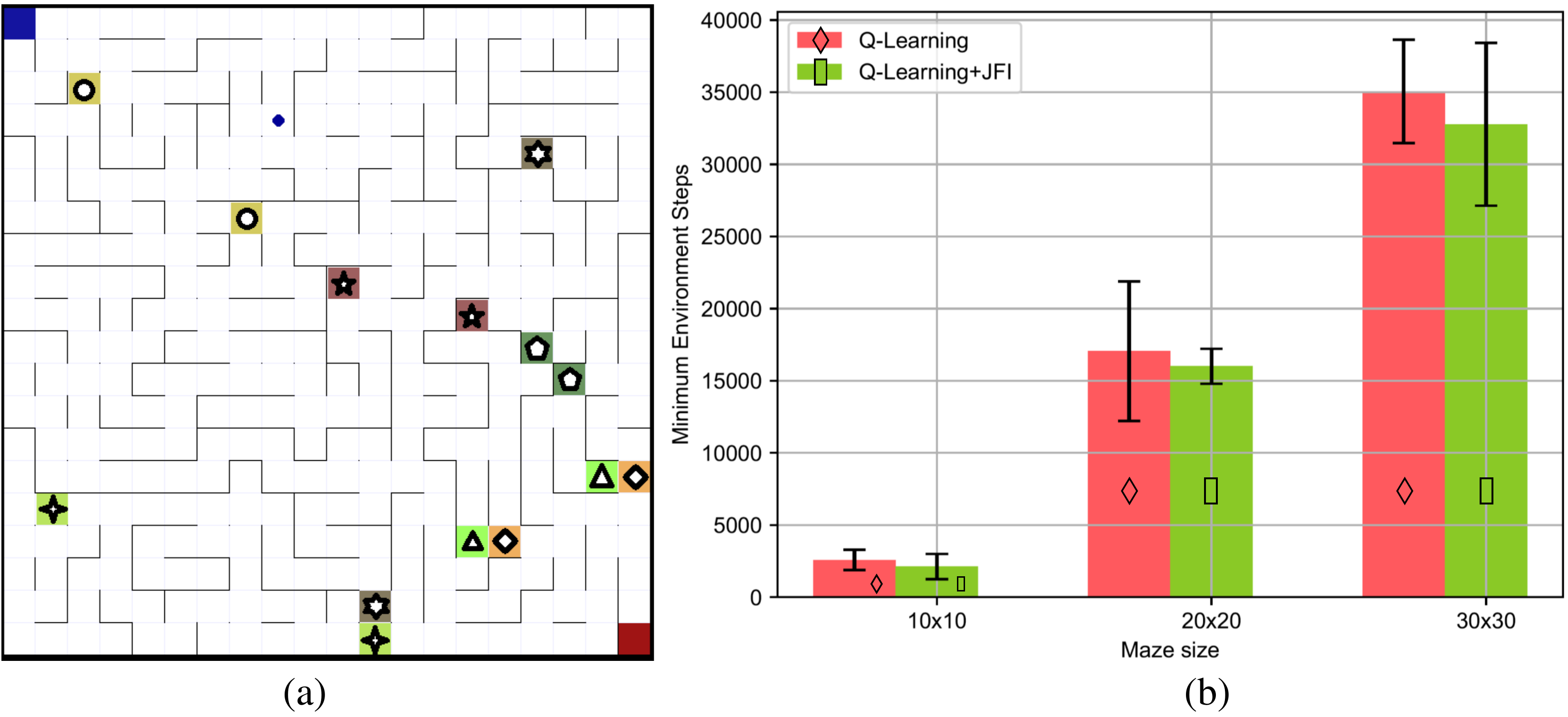}
	\caption{(a) Sample screen of a maze game with size $20\times20$. (b) Average exploration performance comparison over $100$ simulations.} 
	\label{fig:maze}
\end{figure}

\subsubsection{Experimental Setting}
We select the vanilla Q-learning algorithm \cite{watkins1992q} as the benchmark method and perform the experiments on three mazes with different sizes, in which the problem-solving complexity increases exponentially with the maze size. In each episode, the maximum environment steps was set as $10M^{2}$, where $M$ is the maze size. We initialized the Q-table with all zeros and updated the Q-table in every step for efficient training, in which the update formulation is:
\begin{equation}
Q(\mathbf{s},\mathbf{a})\leftarrow Q(\mathbf{s},\mathbf{a})+\alpha [r+\gamma \max_{\mathbf{a}}Q(\mathbf{s},\mathbf{a})-Q(\mathbf{s},\mathbf{a})],
\end{equation}
where $Q(\mathbf{s},\mathbf{a})$ is the action-value function, and $\alpha$ is a learning rate. . The learning rate was set as $0.2$, and a $\epsilon$-greedy policy with an exploration rate $0.001$ was also employed.

\subsubsection{Performance Comparison}
To compare the exploration performance, we choose the minimum environment steps for successfully visiting all states as the key performance indicator (KPI). For instance, a $10\times10$ maze has $100$ grids that corresponds to $100$ states. The minimal steps for the agent to visit all the possible states is measured as its exploration performance. As can be seen in Fig.~\ref{fig:maze}(b), the \textit{Q-learning+JFI} realizes higher performance in all three maze games, which demonstrates the powerful capability of the fairness-driven exploration.

\subsection{Atari Games}
\subsubsection{Experimental Setting}

\begin{table*}[htp]
	\centering
	\caption{Performance comparison in twenty Atari games.}
	\label{tb:dis final performance}
	\resizebox{\textwidth}{40mm}{
		\begin{tabular}{l|l|l|l|l|l}
			\hline
			Game            & PPO  & PPO+RE3  & PPO+RIDE & PPO+RND & PPO+MMRS \\ \hline
			Alien           & 1.84k$\pm$92.78 & 1.99k$\pm$31.82 & 1.88k$\pm$52.43 & 1.86k$\pm$0.11k & ${\bf 2.10k\pm71.71}$ \\
			Air Raid        & 7.57k$\pm$0.71k & 8.30k$\pm$0.61k & 8.12k$\pm$0.60k & 7.97k$\pm$0.79k & ${\bf 8.94k\pm0.73k}$ \\
			Assault         & 4.44k$\pm$0.34k & 5.34k$\pm$0.59k & 5.13k$\pm$0.30k & 4.74k$\pm$0.51k & ${\bf 5.44k\pm0.72k}$ \\		
			Asteroids 	    & 1.86k$\pm$94.36 & 2.27k$\pm$76.15 & 2.15k$\pm$0.15k & 2.27k$\pm$72.46 & ${\bf 2.45k\pm0.21k}$ \\
			Battle Zone     & 18.94k$\pm$5.29k & 29.78k$\pm$1.03k & 26.74$\pm$1.94k & 28.00k$\pm$2.12k & ${\bf 33.63k\pm1.40k}$ \\ 
			Beam Rider      & 2.03k$\pm$0.25k & ${\bf 5.57k\pm0.57k}$ & 2.28k$\pm$0.23k & 2.07k$\pm$0.50k & 2.45k$\pm$0.29k \\
			Breakout        & 0.36k$\pm$40.82 & 0.34k$\pm$38.28 & 0.36k$\pm$31.27 & 0.24k$\pm$11.60 & ${\bf 0.38k\pm12.95}$ \\
			Centipede       & 5.01k$\pm$0.31k & ${\bf 9.51k\pm1.02k}$ & 7.20k$\pm$0.45k & 8.65k$\pm$1.17k & 5.95k$\pm$0.28k \\ 
			Demon Attack    & 7.96k$\pm$0.37k & 24.53k$\pm$3.93k & 22.74k$\pm$2.07k & 4.91k$\pm$11.15k & ${\bf 26.08k\pm3.41k}$ \\
			Frostbite       & 1.27k$\pm$0.13k & 1.28k$\pm$0.11k & 1.36k$\pm$0.12k & 1.36k$\pm$0.22k & ${\bf 1.89k\pm0.26k}$ \\
			Gopher          & 4.20k$\pm$0.30k & 3.79k$\pm$0.16k & 3.69k$\pm$0.33k & ${\bf 4.76k\pm0.87k}$ & 4.00k$\pm$0.23k \\ 
			Gravitar        & 0.38k$\pm$50.67 & 0.40k$\pm$36.61 & 0.62k$\pm$67.74 & 0.51k$\pm$22.53 & ${\bf 0.67k\pm26.34}$ \\ 
			Jamesbond       & 1.79k$\pm$0.40k & 2.74k$\pm$0.54k & 2.38k$\pm$0.36k & 2.05k$\pm$0.88k & ${\bf 2.80k\pm0.40k}$ \\ 
			Krull           & 7.67k$\pm$0.13k & 9.85k$\pm$0.16k & 10.49k$\pm$0.10k & 10.19k$\pm$0.23k & ${\bf 10.58k\pm0.20k}$ \\ 
			Kung Fu Master  & 38.96$\pm$3.29k & 20.49k$\pm$0.81k & 31.54k$\pm$3.47k & 15.94k$\pm$0.98k & ${\bf 44.16k\pm2.49k}$ \\
			Ms Pacman       & 2.53k$\pm$0.12k & 1.78k$\pm$0.12k & 2.72k$\pm$0.22k & 1.56k$\pm$82.7 & ${\bf 2.86k\pm0.25k}$ \\ 
			Phoenix         & 8.22k$\pm$0.82k & 10.90k$\pm$1.37k & 11.19k$\pm$1.03k & 5.75$\pm$0.89k & ${\bf 13.44k\pm0.94k}$ \\ 
			Riverraid       & 8.58k$\pm$0.13k & 5.51k$\pm$0.44k & 8.15k$\pm$0.32k & 4.56k$\pm$0.17k & ${\bf 9.22k\pm0.28k}$ \\ 
			Seaquest        & 0.94k$\pm$4.84 & 1.83k$\pm$58.54 & 1.74k$\pm$42.37 & 0.89k$\pm$6.61 & ${\bf 1.84k\pm75.12}$ \\
			Space Invaders  & 1.03k$\pm$79.63 & ${\bf 1.53k\pm0.19k}$ & 1.28k$\pm$11.93 & 1.06k$\pm$0.80k & 1.10k$\pm$23.38 \\ \hline
	\end{tabular}}
\end{table*}

\begin{figure*}[h]
	\centering
	\includegraphics[width=1.\linewidth]{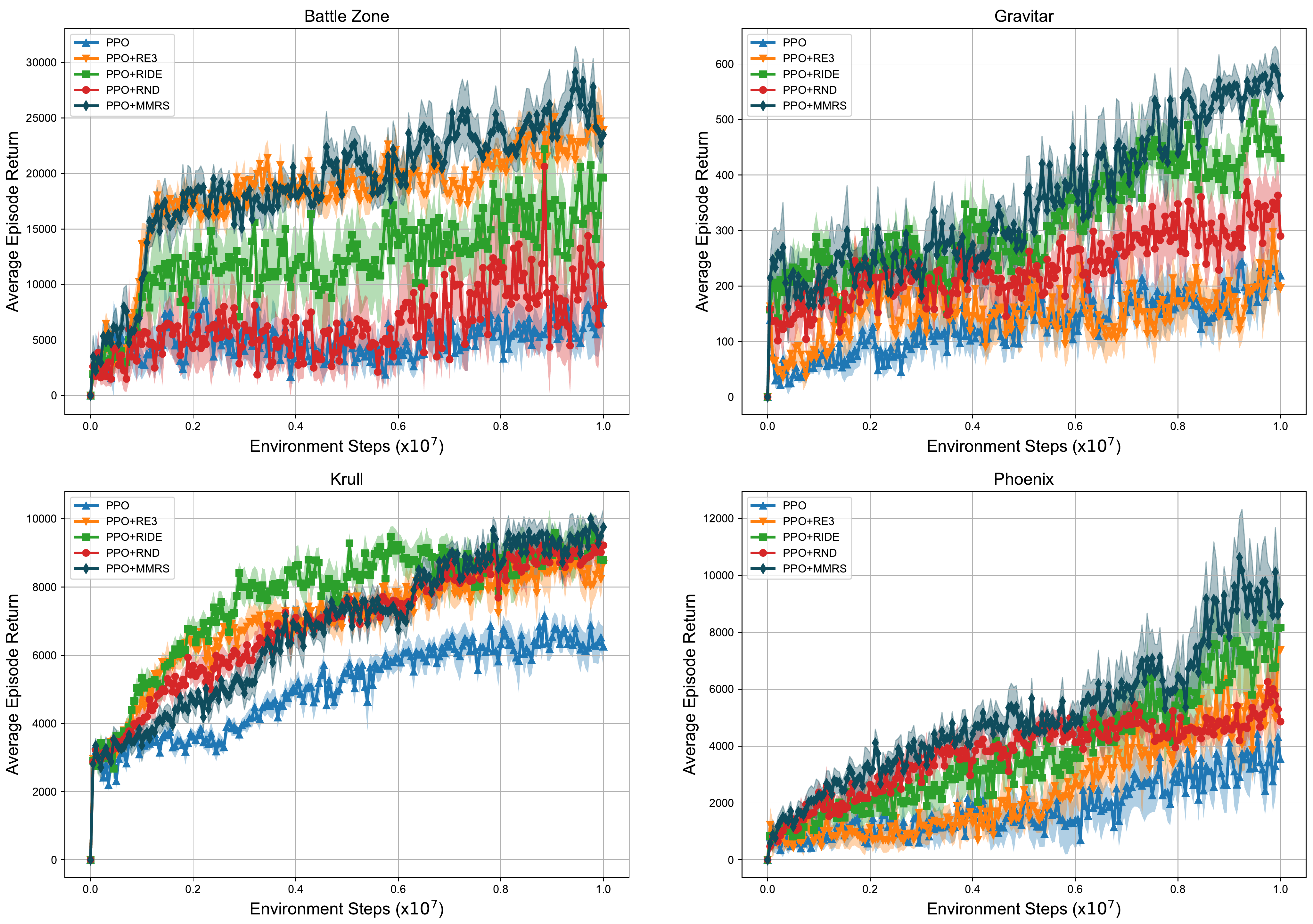}
	\caption{Moving average of episode return versus number of environment steps on Atari games.}
	\label{fig:dis eps return}
\end{figure*}

We next tested MMRS on Atari games \cite{brockman2016openai} with discrete action space, in which the player aims to obtain as higher game points as possible while keeping alive. To generate the observation of the agent, we stacked four consecutive frames as an input. These frames were also resized with shape $(84, 84)$ to reduce computational complexity. In particular, some selected games have complex action space, which can be used to evaluate the generalization ability of MMRS.

To handle the graphic observations, we leveraged convolutional neural networks (CNNs) to build MMRS and benchmark algorithms. The local intrinsic reward block of MMRS needs to learn an encoder and a decoder. The encoder was composed of four convolutional layers and one dense layer, in which each convolutional layer is followed by a batch normalization (BN) layer \cite{ioffe2015batch}. For the decoder, it utilized four deconvolutional layers to perform upsampling. Moreover, a dense layer and a convolutional layer were employed at the top and the bottom of the decoder. Note that no BN layer is included in the decoder. Finally, we used LeakyReLU activation function both for encoder and decoder, and more detailed network architectures are illustrated in Appendix \ref{appendix:atari na}.

We trained MMRS with ten million environment steps. In each episode, the agent was set to interact with eight parallel environments with different random seeds. Moreover, one episode had a length of 128 steps, producing 1024 pieces of transitions. For an observed state, it was first processed by the encoder of VAE to generate a latent vector. Then the latent vector was sent to the decoder to reconstruct the state. The pixel values of the true state and the reconstructed state were normalized into $[-1, 1]$, and the reconstruction error was utilized as the state novelty. After that, we performed $k$-means clustering on the latent vectors with $k=10$. Note that the number of clusters can be bigger if a longer episode length is employed. 

Equipped with the multimodal intrinsic rewards ($\lambda_{G}=\lambda_{L}=0.1$), we used a proximal policy optimization (PPO) \cite{schulman2017proximal} method to update the policy network. More specifically, we used a PyTorch implementation of the PPO method, which can be found in \cite{kostrikov2018github}. To make a fair comparison, we employed an identical policy network and value network for all the algorithms, and its architectures can be found in Appendix \ref{appendix:atari na}. The PPO was trained with a learning rate of $0.0025$, an entropy coefficient of $0.01$, a value function coefficient of $0.5$, and a GAE parameter of $0.95$ \cite{schulman2015high}. In particular, a gradient clipping operation with threshold $[-5, 5]$ was performed to stabilize the learning procedure. 

After the policy was updated, the transitions were utilized to update our VAE model. For the hyper-parameters setting, the batch size was set as 64, and the Adam optimizer was leveraged to perform the gradient descent. In particular, a linearly decaying learning rate was employed to prevent the life-long state novelty from diminishing rapidly. Finally, we trained the benchmark schemes following its default settings reported in the literature.

\subsubsection{Performance Comparison}

For performance comparison, the average one-life return is utilized as the KPI. Table \ref{tb:dis final performance} illustrates the performance comparison over eight random seeds on twenty Atari games, in which the highest performance is shown in bold numbers. As shown in Table \ref{tb:dis final performance}, MMRS achieved the highest performance in sixteen games. RE3 achieved the highest performance in three games, while RND beat all the other methods in one game. RE3 and RIDE outperform the vanilla PPO agent in fifteen games and nineteen games, respectively. In contrast, RND successfully outperformed the vanilla PPO agent in fourteen games. Furthermore, Fig.~\ref{fig:dis eps return} illustrates the moving average of the episode return of four selected games during training. It is clear that the growth rate of MMRS is faster than the other benchmarks, and less oscillation has occurred in the learning procedure.

%

\subsection{Bullet Games}
\begin{table}[h]
	\centering
	\caption{The details of Bullet games with continuous action space.}
	\label{tb:con games}
	\begin{tabular}{l|l|l|l}
		\hline
		Game     & Observation shape & Action extent & Action shape \\ \hline
		Ant      & (28, )            & (-1.0, 1.0)   & (8, )        \\
		Half Cheetah & (26, )        & (-1.0, 1.0)   & (6, )        \\
		Hopper   & (15, )            & (-1.0, 1.0)   & (3, )        \\
		Humanoid & (44, )            & (-1.0, 1.0)   & (17, )       \\
		Inverted Pendulum   & (5, )            & (-1.0, 1.0)   & (1, )        \\
		Walker 2D & (22, )            & (-1.0, 1.0)   & (6, )       \\
		\hline
	\end{tabular}
\end{table}
\subsubsection{Experimental Setting}
We finally tested the MMRS on Bullet games \cite{coumans2016pybullet} with continuous action space, and the details of selected games is shown in Table \ref{tb:con games}. In all six games, the target of the agent is to move forward as fast as possible without falling to the ground. Unlike the Atari games that has graphic observations, Bullet games use fixed-length vectors as observations. For instance, "Ant" game uses 28 features to describe the state of agent, and its action is a vector consists of 8 values within $[-1.0, 1.0]$.

\begin{figure}[h]
	\centering
	\includegraphics[width=1.\linewidth]{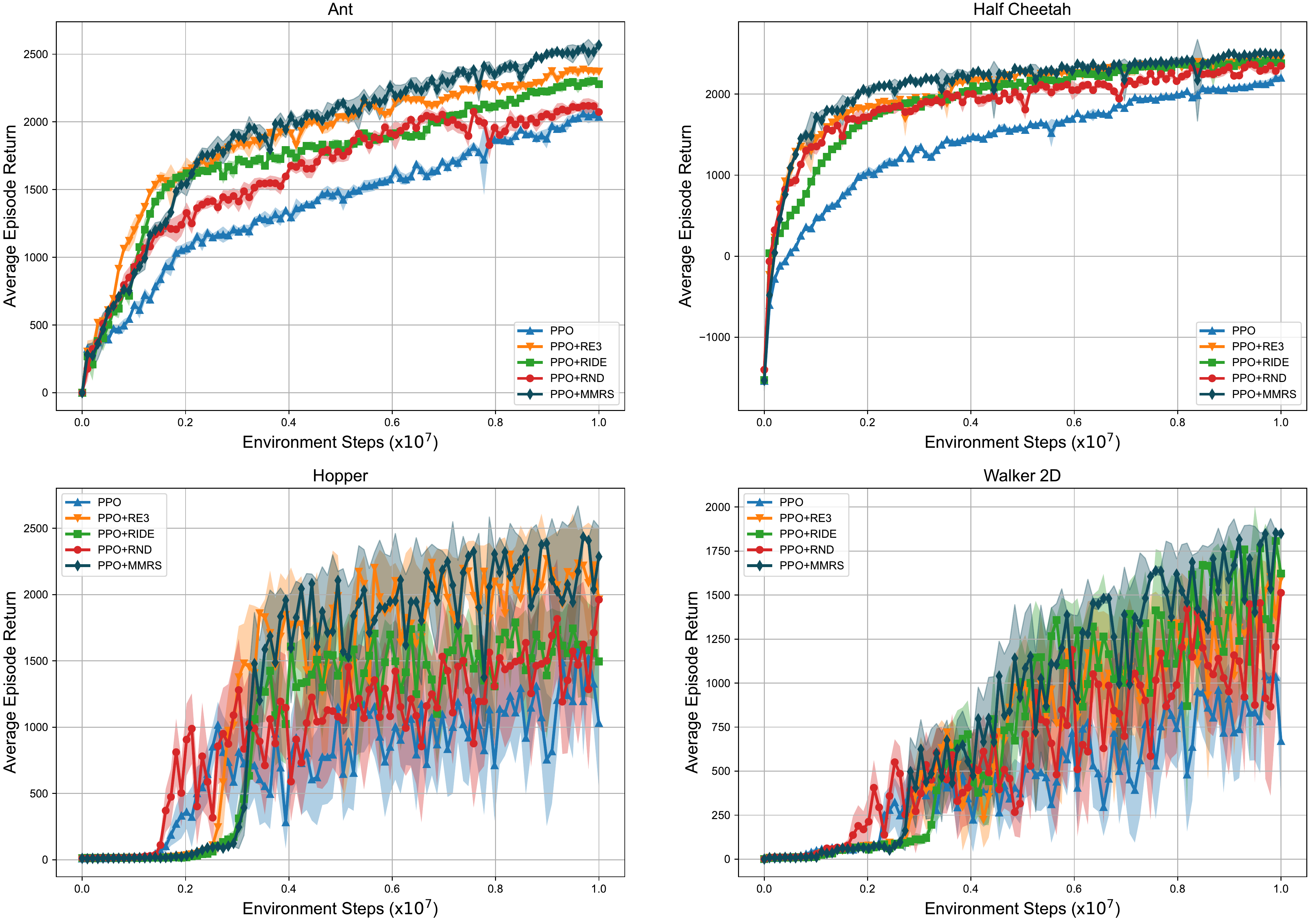}
	\caption{Moving average of episode return versus number of environment steps on Bullet games.}
	\label{fig:con eps return}
\end{figure}

\begin{table*}[h]
	\centering
	\caption{Performance comparison of six Bullet games.}
	\label{tb:con final performance}
	\resizebox{\textwidth}{15mm}{
		\begin{tabular}{l|l|l|l|l|l}
			\hline
			Game     & PPO     & PPO+RE3 & PPO+RIDE & PPO+RND & PPO+MMRS \\ \hline
			Ant               & 2.11k$\pm$42.46  & 2.43k$\pm$22.51  & 2.34k$\pm$14.84  & 2.19k$\pm$18.85  &  ${\bf 2.59k\pm21.44}$  \\
			Half Cheetah      & 2.23k$\pm$12.72 &  2.46k$\pm$6.79  & 2.44k$\pm$12.81  & 2.48k$\pm$48.97  &  ${\bf 2.54k\pm12.33}$  \\
			Hopper   		  & 1.68k$\pm$32.04  & 2.40k$\pm$16.36  & 1.84k$\pm$9.60  &  2.10k$\pm$40.96  &  ${\bf 2.46k\pm12.73}$  \\
			Humanoid          &  1.13k$\pm$80.65 & 1.28k$\pm$0.10k  & 1.15k$\pm$0.21k &  1.09k$\pm$0.13k  &  ${\bf 1.41\pm0.25k}$  \\
			Inverted Pendulum              & 1.00k$\pm$0.00  & 1.00k$\pm$0.00  & 1.00k$\pm$0.00  & 1.00k$\pm$0.00 &  ${\bf 1.00k\pm0.00}$  \\
			Walker 2D         & 1.14k$\pm$31.55  & 1.63k$\pm$12.80  & 1.84k$\pm$14.05  & 1.54k$\pm$32.05  &  ${\bf 1.87k\pm14.65}$  \\
			\hline
	\end{tabular}}
\end{table*}

We leveraged multilayer perceptron (MLP) to implement MMRS and benchmarks, and the detailed network architectures are illustrated in Appendix \ref{appendix:bullet na}. Note that no BN layers were introduced in this experiment. We trained MMRS with ten million environment steps. The agent was also set to interact with eight parallel environments with different random seeds in each episode, and a diagonal Gaussian distribution was used to sample actions. The rest of the updating procedure was consistent with the experiments of Atari games, but no normalization was performed to the states. For computing the multimodal intrinsic rewards, the coefficients were set as $\lambda_{G}=0.01, \lambda_{L}=0.001$.

\subsubsection{Performance Comparison}

Table~\ref{tb:con final performance} illustrates the performance comparison between MMRS and benchmarks, and MMRS achieved the best performance in all six games. Furthermore, Fig.~\ref{fig:con eps return} demonstrates the moving average of episode return of four selected games during the training. It is obvious that MMRS realizes stable and efficient growth when compared with benchmarks. In summary, MMRS shows great potential for obtaining considerable performance both in discrete and continuous control tasks.

\section{Conclusion}
In this paper, we have investigated the problem of improving exploration in RL. We first dived into the sample complexity of the entropy-based approaches and obtained an exact lower bound. To eliminate the prohibitive sample complexity, a novel metric entitled JFI was introduced to replace the entropy regularizer. Moreover, we further proved the utility consistency between the JFI and entropy regularizer, and demonstrated the practical usage of JFI both in tabular setting and infinite state space. Equipped with the JFI metric, the state novelty was integrated to build multimodal intrinsic rewards, which evaluates the exploration extent more precisely. In particular, we used VAE model to capture the life-long state novelty across episodes, it avoids overfitting and learns excellent state representation when compared with the discriminative models. Finally, extensive simulations were performed both in discrete and continuous tasks of Open AI Gym library. The numerical results demonstrated that our algorithm outperformed the benchmarks, showing great effectiveness for realizing efficient exploration.

\section*{Acknowledgements}
This work was supported, in part, by the Shenzhen Institute of Artificial Intelligence and Robotics for Society (AIRS) under grant No. AC01202005001, the Shenzhen Science and Technology Innovation Committee under Grant No. JCYJ20190813170803617. Corresponding author: SimonPun@cuhk.edu.cn.

\newpage
\bibliographystyle{unsrt}

\newpage
\appendix
\onecolumn
\section{Proof of Lemma \ref{lemma:estimate svd}}\label{proof:estimate svd}
Considering sampling for $T$ steps and collecting a dataset $\{\mathbf{s}_{t}\}_{t=0}^{T}$, then a reasonable estimate of the state distribution is:
\begin{equation}\nonumber
	P(\mathbf{s})=\frac{1}{T}\sum_{t=0}^{T}\mathbbm{1}(S_t=\mathbf{s}).
\end{equation}
According to the McDiarmid's inequality \cite{mcdiarmid1989method}, we have that $\forall \epsilon>0$:
\begin{equation}\nonumber
	P\bigg(\Vert P(\mathbf{s})-P(\mathbf{s}) \Vert_{1}\geq\sqrt{|\mathcal{S}|}(\sqrt{\frac{1}{T}}+\epsilon)\bigg)\leq e^{-T\epsilon^{2}}.
\end{equation}
Take logarithm on both sides, such that:
\begin{equation}\nonumber
	\log P \leq -T\epsilon^{2} \Longrightarrow \epsilon \leq \sqrt{\frac{\log (1/P)}{T}}.
\end{equation}
Assume the $P\leq\delta$, so with probability $1-\delta$, it holds:
\begin{equation}\nonumber
	\Vert d^{\pi}(\mathbf{s})-\hat{d}^{\pi}(\mathbf{s}) \Vert_{1} \leq  \sqrt{|\mathcal{S}|}(\sqrt{\frac{1}{T}}+\epsilon) \leq \sqrt{|\mathcal{S}|}(\sqrt{\frac{1}{T}}+\sqrt{\frac{\log (1/P)}{T}}).
\end{equation}

Let $P=\delta, \exists c \in \mathbb{R}^{+}$, such that:
\begin{equation}\nonumber
	\Vert P(\mathbf{s})-\hat{P}(\mathbf{s}) \Vert_{1} \leq c\sqrt{\frac{|\mathcal{S}|\log(1/\delta)}{T}}.
\end{equation}
This concludes the proof.

\section{Proof of Lemma \ref{lemma:estimate entropy}}\label{proof:estimate entropy}

Considering sampling for $T$ steps and collecting a dataset $\{\mathbf{s}_{t}\}_{t=0}^{T}$, then a reasonable estimate of state entropy is:
\begin{equation}\nonumber
	\hat{H}(d^{\pi})=-\frac{1}{T}\sum_{t=0}^{T}\log d^{\pi}(\mathbf{s}).
\end{equation}
For $\forall \epsilon>0$, the Hoeffding's inequality \cite{hoeffding1994probability} indicates that:
\begin{equation}\nonumber
	P\bigg(|\hat{H}(d^{\pi})-H(d^{\pi})|\geq \epsilon \bigg)\leq 2e^{\frac{-2T\epsilon^{2}}{|\mathcal{S}|^{2}}}.
\end{equation}
Therefore,
\begin{equation}\nonumber
	2e^{\frac{-2T\epsilon^{2}}{|\mathcal{S}|^{2}}}\geq 1-P.
\end{equation}
Let $1-P=\delta$, take logarithm on both sides, such that:
\begin{equation}\nonumber
	\log 2 - \frac{2T\epsilon^{2}}{|\mathcal{S}|^{2}} \geq 1-P \Longrightarrow \epsilon \leq |\mathcal{S}|\sqrt{\frac{\log(2/\delta)}{2T}}.
\end{equation}
Finally, it holds:
\begin{equation}\nonumber
	|\hat{H}(d^{\pi})-H(d^{\pi})|\leq \log |\mathcal{S}| \sqrt{\frac{\log (2/\delta)}{2T}}.
\end{equation}
This concludes the proof.

\section{Benchmark Schemes}\label{appendix:benchmarks}

\subsection{RE3}
Given a trajectory $\tau=(\mathbf{s}_{0},\mathbf{a}_{0},\dots,\mathbf{a}_{T-1},\mathbf{s}_{T})$, RE3 first uses a randomly-initialized DNN to encode the visited states. Denote by $\{\mathbf{x}_{i}\}_{i=0}^{T}$ the encoding vectors, RE3 estimates the entropy of $P(\mathbf{s})$ using a $k$-nearest neighbor ($k$-NN) entropy estimator \cite{singh2003nearest}:
\begin{equation}\label{eq:knn ee}
	\begin{aligned}
		\hat{H}(P(\mathbf{s})) &= \frac{1}{T}\sum_{i=0}^{T}\log \frac{T\cdot \Vert \mathbf{x}_{i} - \mathbf{x}_{i}^{k} \Vert_{2}^{C_{\rm e}} \cdot C_{\pi}}{k\cdot\Gamma(\frac{C_{\rm e}}{2}+1)}+\log k-\Psi(k) \\
		&\propto \frac{1}{T}\sum_{i=0}^{T}\vert \mathbf{x}_{i} - \mathbf{x}_{i}^{k} \Vert_{2},
	\end{aligned}
\end{equation}
where $\mathbf{x}_{i}^{k}$ is the $k$-NN of $\mathbf{x}_{i}$ within the set $\{\mathbf{x}_{i}\}_{i=0}^{T}$, $C_{e}$ is the dimension of the encoding vectors, $C_{\pi}\approx3.14159$, $\Gamma(\cdot)$ is the Gamma function, and $\Psi(\cdot)$ is the diagamma function. Equipped with Eq.~\eqref{eq:knn ee}, the intrinsic reward for each transition is computed as:
\begin{equation}
	r^{\rm intrinsic}=\log(\Vert \mathbf{x}_{i} - \mathbf{x}_{i}^{k} \Vert_{2} + 1).
\end{equation}

\subsection{RIDE}
RIDE inherits the architecture of intrinsic curiosity module (ICM) in \cite{pathak2017curiosity}, which is composed of a embedding module $g_{\rm emb}$, an inverse dynamic model $g_{\rm inv}$, and a forward dynamic model $g_{\rm for}$. Given a transition $(\mathbf{s},\mathbf{a},\mathbf{s}')$, the inverse dynamic model predicts an action using the encoding of state $\mathbf{s}$ and next-state $\mathbf{s}'$. Meanwhile, the forward dynamic model accepts $g_{\rm emb}(\mathbf{s})$ and the true action $\mathbf{a}$ to predict the representation of $\mathbf{s}'$. Given a trajectory $\tau=(\mathbf{s}_{0},\mathbf{a}_{0},\dots,\mathbf{a}_{T-1},\mathbf{s}_{T})$, the three models are trained to minimize the following loss function:
\begin{equation}
	L_{\rm RIDE}=\sum_{t=0}^{T}\Vert g_{\rm for}\big(g_{\rm emb}(\mathbf{s}_{t}),\mathbf{a}_{t}\big)-g_{\rm emb}(\mathbf{s}_{t+1}) \Vert_{2}^{2}+L_{\rm inv}\bigg(\mathbf{a}_{t}, g_{\rm inv}\big(g_{\rm emb}(\mathbf{s}_{t}),g_{\rm emb}(\mathbf{s}_{t+1})\big)\bigg),
\end{equation}
where $L_{\rm inv}$ denotes the loss function that measures the distance between true actions and predicted actions, e.g., the cross entropy for discrete action space. Finally, the intrinsic reward of each transition is computed as:
\begin{equation}
	r^{\rm intrinsic}=\frac{\Vert g_{\rm emb}(\mathbf{s}_{t+1}) - g_{\rm emb}(\mathbf{s}_{t}) \Vert_{2}}{\sqrt{N_{\rm eps}(\mathbf{s}_{t+1})}},
\end{equation}
where $N_{\rm eps}(\mathbf{s}_{t+1})$ is the number of times that state has been visited during the current episode, it can be obtained using pseudo-count method \cite{ostrovski2017count}.

\subsection{RND}
RND leverages a predictor network and a target network to record the visited states. The target network serves as the reference, which is fixed and randomly-initialized to set the prediction problem. The predictor network is trained using the collected data by the agent across the episodes. Denote by $h:\mathcal{S}\rightarrow \mathbb{R}^{m}$ and $\hat{h}:\rightarrow \mathbb{R}^{m}$ the target network and predictor network, where $m$ is the embedding dimension. The RND is trained to minimize the following loss function:
\begin{equation}
	L_{\rm RND}=\sum_{t=0}^{T}\Vert \hat{h}(\mathbf{s}_{t})-h(\mathbf{s}_{t}) \Vert_{2}^{2}.
\end{equation}
Finally, the intrinsic reward of each transition is computed as:
\begin{equation}
	r^{\rm intrinsic}=\Vert \hat{h}(\mathbf{s}_{t+1})-h(\mathbf{s}_{t+1}) \Vert_{2}^{2}.
\end{equation}

\section{Network Architecture}
\subsection{Atari Games}\label{appendix:atari na}

For instance, "8$\times$8 Conv. 32" represents a convolutional layer that has 32 filters of size 8$\times$8. A categorical distribution was used to sample an action based on the action probability of the stochastic policy. Note that "Dense 512 \& Dense 512" in Table \ref{tb:cnn na} means that there are two branches for outputing the mean and variance of the latent variables, respectively.
\begin{table}[h]
	\centering
	\caption{The CNN-based network architectures.}
	\label{tb:cnn na}
	\begin{tabular}{l|l|l|l}
		\hline
		Moudle & Policy network $\pi$                                                                                                                                                                                          & Encoder $p_{\bm \theta}$                                                                                                                                                                                                       & Decoder $q_{\bm \phi}$                                                                                                                                                                                                                                                      \\ \hline
		Input  & State                                                                                                                                                                                                    & State                                                                                                                                                                                                                 & Latent Variables                                                                                                                                                                                                                                    \\ \hline
		Arch.  & \begin{tabular}[c]{@{}l@{}}8$\times$8 Conv. 32, ReLU\\ 4$\times$4 Conv. 64, ReLU\\ 3$\times$3 Conv. 32, ReLU\\ Flatten\\ Dense 512, ReLU\\ Dense $|\mathcal{A}|$\\ Categorical Distribution\end{tabular} & \begin{tabular}[c]{@{}l@{}}3$\times$3 Conv. 32, LeakyReLU\\ 3$\times$3 Conv. 32, LeakyReLU\\ 3$\times$3 Conv. 32, LeakyReLU\\ 3$\times$3 Conv. 32\\ Flatten\\ Dense 512 \& Dense 512\\ Gaussian sampling\end{tabular} & \begin{tabular}[c]{@{}l@{}}Dense 64, LeakyReLU\\ Dense 1024, LeakyReLU\\ Reshape\\ 3$\times$3 Deconv. 64, LeakyReLU\\ 3$\times$3 Deconv. 64, LeakyReLU\\ 3$\times$3 Deconv. 64, LeakyReLU\\ 8$\times$8 Deconv. 32\\ 1$\times$1 Conv. 4\end{tabular} \\ \hline
		Output & Action                                                                                                                                                                                                   & Latent variables                                                                                                                                                                                                      & Reconstructed state                                                                                                                                                                                                                                 \\ \hline
	\end{tabular}
\end{table}
\subsection{Bullet Games}\label{appendix:bullet na}
\begin{table}[h]
	\centering
	\caption{The MLP-based network architectures.}
	\label{tb:mlp na}
	\begin{tabular}{l|l|l|l}
		\hline
		Moudle & Policy network $\pi$                                                                                                                                                                                          & Encoder $p_{\bm \theta}$                                                                                                                                                                                                       & Decoder $q_{\bm \phi}$                                                                                                                                                                                                                                                      \\ \hline
		Input  & State                                                                                                                                                                                                    & State                                                                                                                                                                                                                 & Latent Variables                                                                                 \\ \hline
		Arch.  & \begin{tabular}[c]{@{}l@{}}Dense 64, Tanh\\ Dense 64, Tanh\\ Dense $|\mathcal{A}|$\\ Categorical Distribution\end{tabular} & \begin{tabular}[c]{@{}l@{}}Dense 32, Tanh\\ Dense 64, Tanh\\ Dense 256\\ Dense 256 \& Dense 512\\ Gaussian sampling\end{tabular} & \begin{tabular}[c]{@{}l@{}}Dense 32, Tanh\\ Dense 64, Tanh\\ Dense observation shape\end{tabular} \\ \hline
		Output & Action                                                                                                                     & Latent variables                                                                                                                 & Reconstructed state                                                                               \\ \hline
	\end{tabular}
\end{table}
	
\end{document}